\documentclass{article}

\usepackage{subfigure}
\usepackage{booktabs} % for professional tables
\usepackage{subcaption}

\usepackage{cancel}
\usepackage{hyperref}

\usepackage[accepted]{icml2025}

\usepackage{amsmath}
\usepackage{amsthm}
\usepackage{thm-restate}
\usepackage{cancel}
\usepackage{graphicx} 
\usepackage[utf8]{inputenc} % allow utf-8 input
\usepackage[T1]{fontenc}    % use 8-bit T1 fonts
\usepackage{hyperref}       % hyperlinks
\usepackage{url}            % simple URL typesetting
\usepackage{booktabs}       % professional-quality tables\textbf{}
\usepackage{amsfonts}       % blackboard math symbols
\usepackage{nicefrac}       % compact symbols for 1/2, etc.
\usepackage{microtype}      % microtypography
\usepackage{relsize}
\usepackage{tikz}
\usepackage{wrapfig}
\usepackage{xcolor}         % colors

\usepackage[capitalize,noabbrev]{cleveref}

\theoremstyle{plain}
\newtheorem{theorem}{Theorem}[section]

\theoremstyle{definition}

\theoremstyle{remark}

\newcommand{\aalpha}{\boldsymbol{\alpha}}
\newcommand{\bbeta}{\boldsymbol{\beta}}

\newcommand{\mmu}{\boldsymbol{\mu}}
\newcommand{\Pphi}{\boldsymbol{\Phi}}
\newcommand{\SSigma}{\boldsymbol{\Sigma}}

\newcommand{\E}{\mathbb{E}}
\newcommand{\R}{\mathbb{R}}

\newcommand{\VV}{\boldsymbol{V}}
\newcommand{\WW}{\boldsymbol{W}}
\newcommand{\bb}{\boldsymbol{b}}
\newcommand{\cc}{\boldsymbol{c}}
\newcommand{\xx}{\boldsymbol{x}}
\newcommand{\yy}{\boldsymbol{y}}
\newcommand{\zz}{\boldsymbol{z}}

\usepackage[textsize=tiny]{todonotes}

\icmltitlerunning{Converting MLPs into Polynomials}

\begin{document}

\twocolumn[
\icmltitle{Converting MLPs into Polynomials in Closed Form}

% It is OKAY to include author information, even for blind
% submissions: the style file will automatically remove it for you
% unless you've provided the [accepted] option to the icml2025
% package.

% List of affiliations: The first argument should be a (short)
% identifier you will use later to specify author affiliations
% Academic affiliations should list Department, University, City, Region, Country
% Industry affiliations should list Company, City, Region, Country

% You can specify symbols, otherwise they are numbered in order.
% Ideally, you should not use this facility. Affiliations will be numbered
% in order of appearance and this is the preferred way.
\icmlsetsymbol{equal}{*}

\begin{icmlauthorlist}
\icmlauthor{Nora Belrose}{eai}
\icmlauthor{Alice Rigg}{eai}
\end{icmlauthorlist}

\icmlaffiliation{eai}{EleutherAI}

\icmlcorrespondingauthor{Nora Belrose}{nora@eleuther.ai}

% You may provide any keywords that you
% find helpful for describing your paper; these are used to populate
% the "keywords" metadata in the PDF but will not be shown in the document
\icmlkeywords{Machine Learning, ICML}

\vskip 0.3in
]

% this must go after the closing bracket ] following \twocolumn[ ...

% This command actually creates the footnote in the first column
% listing the affiliations and the copyright notice.
% The command takes one argument, which is text to display at the start of the footnote.
% The \icmlEqualContribution command is standard text for equal contribution.
% Remove it (just {}) if you do not need this facility.

\printAffiliationsAndNotice{}  % leave blank if no need to mention equal contribution
%\printAffiliationsAndNotice{\icmlEqualContribution} % otherwise use the standard text.

\begin{abstract}
Recent work has shown that purely quadratic functions can replace MLPs in transformers with no significant loss in performance, while enabling new methods of interpretability based on linear algebra. In this work, we theoretically derive closed-form least-squares optimal approximations of feedforward networks (multilayer perceptrons and gated linear units) using polynomial functions of arbitrary degree. When the $R^2$ is high, this allows us to interpret MLPs and GLUs by visualizing the eigendecomposition of the coefficients of their linear and quadratic approximants. We also show that these approximants can be used to create SVD-based adversarial examples. By tracing the $R^2$ of linear and quadratic approximants across training time, we find new evidence that networks start out simple, and get progressively more complex. Even at the end of training, however, our quadratic approximants explain over 95\% of the variance in network outputs.

\end{abstract}

\section{Introduction}

% Current writing fits well with TMS and if this paper included the skip-transcoder application. Does not support training dynamics sections. 
In the field of mechanistic interpretability, it is well-understood that MLP neurons tend to be \emph{polysemantic} in the sense that they activate on a set of diverse, seemingly unrelated contexts \cite{elhage2022toy}. To address this, the current paradigm is dictionary learning: training wide, sparsely activating MLPs that minimize mean squared-error with respect to the original activations. While sparse autoencoders (SAEs) can be used to extract interpretable features from MLP activations and outputs \cite{huben2023sparse, paulo2024automatically}, %they require a lot of compute to train, Further, 
it is not understood how much they are learning features of the model versus features of the data. % they require a lot of compute to train, and cannot perfectly reconstruct the original activations.

How can we learn features of the model itself, making minimal assumptions about the input data? We propose to assume the data distribution is maximum entropy, subject to low-order constraints on its statistical moments. Famously, for vector-valued data with known mean $\mmu$ and covariance matrix $\SSigma$, this is the Gaussian distribution $\mathcal{N}(\mmu,\SSigma)$. %What architecture should be used to extract the features, and how can it be efficiently fit?
%\textcolor{blue}{Need a transition sentence or two: Questions remain such as: what kind of architecture to fit? You could do it via Monte Carlo sampling or gradient descent. But we directly solve for the optimum.}

% However, modern transformer architectures now use variants of the gated linear unit (GLU) rather than the MLP after each attention block \cite{shazeer2020glu}. The simplest GLU variant is the ``bilinear'' layer \cite{mnih2007three, dauphin2017language}, with the functional form
% \begin{equation}
%    \mathrm{Bilinear}(x, W, V, b, c) = (xW + b) \otimes (xV + c)
% \end{equation}
% where $\otimes$ denotes elementwise multiplication.

%Recently, \citet{pearce2024bilinear} showed that each output unit of the bilinear layer is a quadratic form in the input, and this fact enables mechanistic interpretability via eigendecomposition of the associated Hessians, which they call ``interaction matrices.'' The top eigenvectors of these matrices are often interpretable by direct inspection, and can be used for steering and adversarial attacks.

% Say somewhere that for N(0,1) fit we don't need any data at all
In this paper, we show how to \emph{analytically convert} pretrained MLPs and GLUs into polynomials that globally minimize mean squared error (MSE) over the maximum entropy distribution $\mathcal{N}(\mmu,\SSigma)$. On high dimensional data, this process is tractable for linear and quadratic approximants, which we show is often sufficient to interpret networks trained on simple image classification data such as MNIST. %whose interaction matrices can be interpreted via eigendecomposition, using only the weights of the hidden layer, and sample mean and covariance statistics \emph{optionally}. Our mathematical framework is very general and also allows for the approximation of MLPs using linear functions, or polynomials of higher degree. 

%A key feature of our approach is that higher order approximations have lower MSE error, and a higher order decomposition includes all lower terms as well. 

% Add DSB and outline contributions
We use these polynomial approximants to shed new light on the inductive biases of neural network architectures. Neural networks have been speculated and confirmed to varying extents that they learn lower order statistics first before moving onto higher orders, a concept termed the (distributional) simplicity bias (DSB) \citep{refinetti2023neural, belrose2024neural}. Using our analytic derivation, we test the DSB hypothesis explicitly by measuring how well the least-squares linear or quadratic function approximates a neural network at different stages of training.

Specifically, we measure the fraction of variance unexplained (FVU) of our linear and quadratic approximants across training time, and uncover a phase transition during training where the FVU of the approximant starts to rise sharply, while the quadratic FVU stays nearly constant. We interpret this as indicating a shift from learning linear features to learning nonlinear, largely quadratic features. This observation is broadly consistent with the DSB hypothesis.

We also demonstrate that our linear approximants can be used to generate adversarial examples for the original network, showing that they effectively capture the out-of-distribution behavior of the network they are fit to.

%In section 3, We illustrate the closed form approximation in synthetic settings where the ground truth features are known: modelling continuous functions, and toy models of superposition. We see ...  

%In section 4, we apply the 
% Motivating the math:
\section{Background}
%\textcolor{blue}{Right now, just dump any and all background. Clean it up later.}
% Recently, \citet{pearce2024bilinear} showed that each output unit of the bilinear layer is a quadratic form in the input, and this fact enables mechanistic interpretability via eigendecomposition of the associated Hessians, which they call ``interaction matrices.'' The top eigenvectors of these matrices are often interpretable by direct inspection, and can be used for steering and adversarial attacks.

%There should be a section where I explain how linear and quadratic networks can be interpreted.
%In the context of neural networks, mechanistic interpretability seeks to explain 

% Basics:
Multi-layer perceptrons (MLPs) have the functional form 
\begin{align}
    f(\xx) &= \phi(\xx \WW_1 + \bb_1) \WW_2 + \bb_2
\end{align} where $\phi$ is a nonlinear activation function applied pointwise. Common choices for the activation function are ReLU and GELU \citep{hendrycks2016gaussian}.

However, many recent transformers use variants of the Gated Linear Unit (GLU) rather than the MLP after each attention block \cite{shazeer2020glu}, with the functional form
\begin{equation}\label{eq:glu}
   \mathrm{GLU}_\phi(\xx, \WW, \VV, \bb, \cc) = \phi(\xx \WW + \bb) \odot (\xx \VV + \cc)
\end{equation}
where $\odot$ denotes elementwise multiplication.  Setting $\phi$ to the identity function yields the the simple ``bilinear'' layer \cite{mnih2007three, dauphin2017language}, which has similar performance to GLUs with nonlinearities like Swish (SwiGLU) or ReLU (ReGLU).

Recently, \citet{pearce2024bilinear}  showed that each output unit of the bilinear layer is a quadratic form in the input, and this fact enables mechanistic interpretability via eigendecomposition of the associated Hessians, which they call ``interaction matrices.''
For each target class $k$, the spectral theorem gives us that an orthonormal set of eigenvectors $\{\textbf{v}_i\}_{i=1}^d$ exists, and
\begin{align*}
    \textbf{e}_k \cdot \mathcal{B} = Q = \sum_i^d \lambda_i \textbf{v}_i \textbf{v}_i^T.
\end{align*}
The highest magnitude eigenvectors of these matrices are often interpretable by direct inspection and can be used for steering and adversarial attacks.

The above approach to interpreting neural networks involves starting with a bilinear or quadratic architecture for the base model, which is an unconventional design choice. We introduce the decomposition here, because our quadratic approximants have tensor parameters amenable to the same spectral decomposition. % training a network to use a bilinear layer from the start, which are somewhat of an unconventional architectural choice. %. Why mention these facts? Our goal is to interpret an MLP. 
%A linear approximant has the form $T(\textbf{x})=\bbeta x+\aalpha$
%$f(\textbf{x})=W_2 \text{ReLU}(W_1\textbf{x}+\textbf{b}_1)+\textbf{b}_2$. 
%Throughout this paper, Superposition 

%Mechanistic analysis of a linear transformation is relatively straightforward, with each class being supported or inhibited by different pixels represented directly in the matrix entries. In the second order case, \citep{pearce2024bilinear} demonstrates that the symmetric quadratic tensor can be studied through eigendecomposition. For each target class $k$, the spectral theorem gives us that an orthonormal set of eigenvectors $\{\textbf{v}_i\}_{i=1}^d$ exists, and
%\begin{align*}
%    \textbf{e}_k \cdot \mathcal{B} = Q = \sum_i^d \lambda_i \textbf{v}_i \textbf{v}_i^T.
%\end{align*}
\section{Derivation}
\label{ssec:derivation}

While stochastic gradient descent could be used to estimate polynomial coefficients that approximate a neural network on an arbitrary input distribution, this can be computationally intensive, and does not afford a deeper theoretical understanding of the network's inductive biases. In this section we show that, surprisingly, it is possible to derive \emph{analytic formulas} for these polynomial approximants when the input is assumed to be drawn from some Gaussian mixture. We start with the single Gaussian case, and extend to the case of general Gaussian mixtures in Section~\ref{sec:mixture}.

\subsection{Linear case}
For simplicity, assume that our approximant is affine, taking the form $g(\xx) = \bbeta^T \xx + \aalpha$. Then our problem reduces to ordinary least squares (OLS):
\begin{equation}
    \mathop{\mathrm{argmin\:}}_{(\bbeta, \aalpha)} \E_{\xx} \| f(\xx) - (\bbeta^T \xx + \aalpha) \|^2_2
\end{equation}
which is known to have the solution:
\begin{align}
    \bbeta &= \mathrm{Cov}[\xx]^{-1} \mathrm{Cov}[f(\xx), \xx]\\
    \aalpha &= \E[f(\xx)] - \bbeta^T \E[\xx]
\end{align}
In order to compute these coefficients in closed form, we must analytically evaluate the integrals $\E[f(\xx)]$ and $\mathrm{Cov}[f(\xx), \xx]$. Our key insight is that this is possible when the input distribution is Gaussian, and $f$ takes the form $f(\xx) = \mathbf{W}_2 \phi( \mathbf{W}_1 \xx + \mathbf{b}_1 ) + \mathbf{b}_2$, where $\phi$ is an elementwise nonlinearity.

Given that $\xx \sim \mathcal{N}(\mmu, \SSigma)$ for some $\mmu \in \R^d$, $\SSigma \in \mathbb{S}^d_+$, the pre-activations $\yy := \mathbf{W}_1 \xx + \mathbf{b}_1$ will also be Gaussian with mean $\mathbf{W}_1 \mmu + \mathbf{b}_1$ and covariance $\mathbf{W}_1 \SSigma \mathbf{W}_1^T$. Then each coordinate of the post-activations $\phi(\yy)$ has a mean which can be evaluated analytically when $\phi$ is $\mathrm{ReLU}$ or $\mathrm{GELU}$ (Appendix~\ref{app:relu}), or approximated to high precision using \href{https://en.wikipedia.org/wiki/Gauss%E2%80%93Hermite_quadrature}{Gauss-Hermite quadrature} otherwise. Given $\E[\phi(\yy)]$, we can simply apply the linearity of expectation to compute $\E[f(\xx)] = \mathbf{W}_2 \E[\phi(\yy)] + \mathbf{b}_2$.

Evaluating $\mathrm{Cov}[f(\xx), \xx] = 
\E[f(\xx) \xx^T] - \E[f(\xx)] \E[\xx]^T$ is somewhat more involved. Again applying linearity we have that $\E[f(\xx) \xx^T] = \mathbf{W}_2 \E[\phi(\yy) \xx^T] + \mathbf{b}_2 \E[\xx]^T$. Since $\xx$ and $\yy$ are jointly Gaussian, we can apply \href{https://en.wikipedia.org/wiki/Stein%27s_lemma}{Stein's lemma} to write the $(i, j)$\textsuperscript{th} entry of the first term as
\begin{equation}
    \E[\phi(y_i) x_j] = \mathrm{Cov}(y_i, x_j) \E[\phi'(y_i)]
\end{equation}
where $\phi'$ is the first derivative of $\phi$. Since $\mathrm{ReLU}$ and $\mathrm{GELU}$ have simple derivatives with analytically tractable Gaussian expectations, we can compute $\E[\phi'(y_i)]$ in closed form.

%\textbf{Note.} The extent to which a gaussian prior is a good proxy for the expected model statistics 

\subsection{Quadratic case}

We can reduce the case where $g$ is a polynomial of degree $n > 1$ to an OLS problem using a classic change of variables strategy. Consider the quadratic feature map
\begin{align}\label{eq:quadratic-map}
    \phi_2(\xx) 
&= \{\, x_i x_j \mid 1 \le i \le j \le n\ \}\\
&= \bigl( x_1^2,\; x_1 x_2,\; \dots,\; x_n^2 \bigr).
\end{align}
Let $\zz$ denote the concatenation of $\xx$ and $\phi_2(\xx)$. Now any quadratic function of $\xx$ can be expressed as an affine function of $\zz$, allowing us to estimate the coefficients of the least-squares quadratic in $\xx$ using OLS on $\zz$. This strategy can be applied, \emph{mutatis mutandis}, to polynomial function classes of arbitrary degree.

Solving the OLS problem in $\zz$ requires evaluating the cross-covariance matrix $\mathrm{Cov}[f(\xx), \zz]$ whose $(i, j)$\textsuperscript{th} entry is a Gaussian integral of the form $\E[\phi(y_i) x_k x_l]$, where $y_i, x_k, x_l$ are jointly Gaussian, and $\phi$ is some nonlinearity. We show how this multivariate integral can be decomposed into a linear combination of univariate integrals in Theorem~\ref{thm:master}.

We will also need to evaluate $\mathrm{Cov}[\zz]$, a matrix containing various higher moments of jointly Gaussian variables. This can be done using the formula
\begin{align*}
    \mathbb{E}&\left[ X_1 X_2 \cdots X_n \right] =\\ &\sum_{\text{partitions } (S, P)} \left( \prod_{s \in S} \mu_s \right) \left( \prod_{(i,j) \in P} \text{Cov}(X_i, X_j) \right),
\end{align*}
where the sum ranges over all possible partitions of the variables into singleton sets $S$ and pairs $P$. This generalizes \citet{isserlis1918formula}'s theorem to noncentral Gaussian variables, and can be derived from the \href{https://en.wikipedia.org/wiki/Cumulant}{cumulant-moment identity}.

\subsection{Master theorem}

The following theorem allows us to reduce a multivariate integral of the form $\E [g(X) \prod_{i = 1}^{n} Y_i ]$, where the variables are all jointly Gaussian, into a linear combination of univariate integrals of the form $\E[g(X) X^k]$. We can then efficiently evaluate these integrals using numerical integration techniques or closed form formulas (Appendix~\ref{app:relu}).

\begin{figure*}[t]
    \centering
    \includegraphics[width=0.9\textwidth]{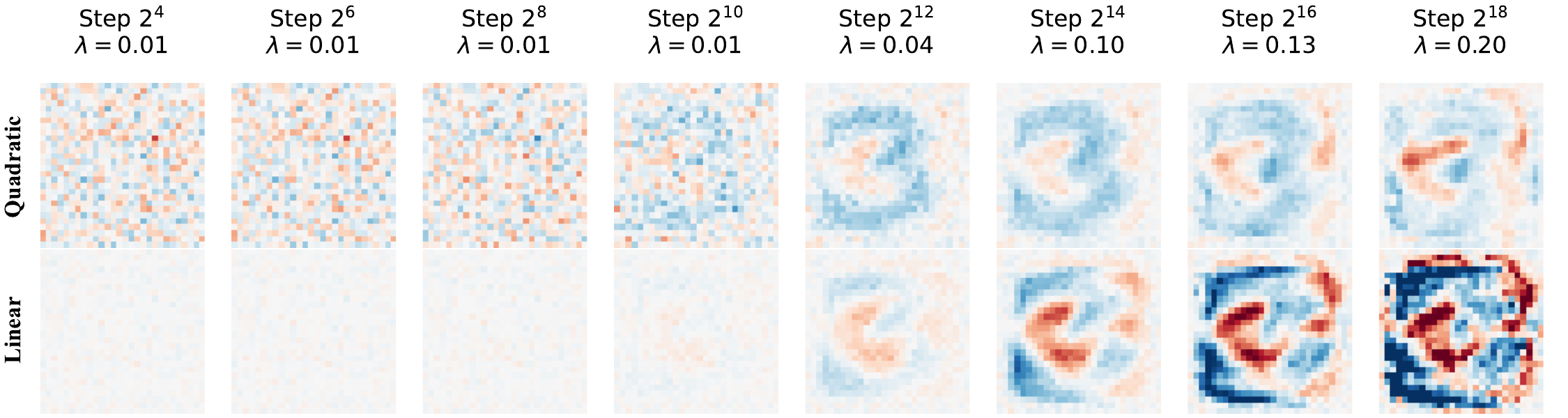}
    \caption{Quadratic and linear features for class `3', over the course of training. Discernible `3' qualities arise from noise for both quadratic and linear features, aligning with the region where learning is happening. The linear 3 structure is most intuitively discernible at step (y), when FVU is minimal, before beginning to overfit. This can be interpreted as the MLP learning and relying on statistics of higher complexity than linear, especially if its accuracy continues to improve. The quadratic feature crystallizes later than linear, and predictably forms visual artifacts at the latest stages of training.}
    \label{fig:qualitative}
\end{figure*}

% Theorem statement
\begin{theorem}[Master Theorem]\label{thm:master}
Let \( X, Y_1, \ldots, Y_n \) be \( n+1 \) jointly Gaussian random variables, and let \( g : \mathbb{R} \rightarrow \mathbb{R} \) be a continuous, real-valued function. Then:
\begin{align}
    \E \Big [g(X) \prod_{i = 1}^{n} Y_i \Big ] = \sum_{k = 0}^{n - 1} a_k \E [ g(X) X^k ],
\end{align}
where the coefficients $a_k$ can be computed analytically in the manner described below.
\end{theorem}
\begin{proof}
We begin by rewriting the $Y_i$ in terms of their conditional expectations on $X$:
\begin{align}\label{eq:setup}
    \E \Big [g(X) \prod_{i = 1}^{n} Y_i \Big ] &= \E \Big [ g(X) \prod_{i = 1}^{n} \big ( \E[Y_i | X] + \epsilon_i \big ) \Big ] \\ &= \E \Big [ g(X) \prod_{i = 1}^{n} \big ( \alpha_i + \beta_i X + \epsilon_i \big ) \Big ],
\end{align}
where $\alpha_i$ and $\beta_i$ are the ordinary least squares intercepts and coefficients, respectively, for regressing $Y_i$ on $X$.\footnote{Here we assume $\mathrm{Var}(X) > 0$, so that these coefficients are well-defined. If $\mathrm{Var}(X) = 0$, then $g(X)$ is almost surely constant, and hence we can write our expectation as $\E[g(X)] \E[\prod_{i = 1}^{n} Y_i]$ and apply Isserlis' theorem to evaluate $\E[\prod_{i = 1}^{n} Y_i]$.}

This leaves us with $g(X)$ times a polynomial in $X$. We will compute the coefficients of this polynomial using a combinatorial argument. First let $S$ be the set of $n$-tuples such that the $i$\textsuperscript{th} entry of each tuple is chosen from $\{ \alpha_i, \beta_i X, \epsilon_i \}$.
\begin{equation}
    S = \big \{ ( d_1, \ldots, d_n) \:|\: \forall i : d_i \in \{ \alpha_i, \beta_i X, \epsilon_i \} \big \}
\end{equation}
Let $\mathrm{prod}(s)$ denote the product of the elements of a tuple $s$. By the distributive property, we can expand our polynomial into a sum of $3^n$ terms, one for each element of $S$:
\begin{equation}\label{eq:distributive}
    \prod_{i = 1}^{n} \big ( \alpha_i + \beta_i X + \epsilon_i \big ) = \sum_{s \in S} \mathrm{prod}(s).
\end{equation}
Now fix some integer $k \leq n$. To compute the polynomial coefficient $a_k$ corresponding to $X^k$, we must sum together all the terms in Eq.~\ref{eq:distributive} which contain precisely $k$ factors drawn from $B = \{ \beta_1, \ldots \beta_n \}$. Let $C(B, k)$ denote the set of $k$-combinations (represented as sorted $k$-tuples) of the elements of $B$, so that $|C(B, k)| = \binom{n}{k}$,\footnote{By construction, the elements of the tuples in $S$ are sorted in ascending order by subscript, so there are only $\binom{n}{k}$ ways for an $\alpha$ factor to appear $k$ times in a term. Since multiplication is commutative, the actual ordering of elements in each tuple is a matter of indifference.} and let $I(c)$ denote the set of indices used in a combination $c$, for example $I ( ( \beta_2, \beta_4 ) ) = \{ 2, 4 \}$. Then we have
\begin{equation}\label{eq:poly-coef}
    a_k = \sum_{c \in C(B, k)} \mathrm{prod}(c) \sum_{s \in S_{-\beta}(c)}  \mathrm{prod}(s),
\end{equation}
where we define $S_{-\beta}(c)$ to be the set of tuples of length $n - k$ whose $i$\textsuperscript{th} element is drawn from $\{ \alpha_i, \epsilon_i \}$:
\begin{align}
    S_{-\beta}(c) = \big \{ ( d_1, \ldots, d_n) \:|\: \forall i \notin I(c) : d_i \in \{ \alpha_i, \epsilon_i \} \big \}
\end{align}
Plugging Eq.~\ref{eq:poly-coef} into Eq.~\ref{eq:setup} yields
\begin{align*}
    & \E \Big [g(X) \prod_{i = 1}^{n} Y_i \Big ] \\
    &= \E \Big [ \sum_{k = 0}^{n} \sum_{c \in C(B, k)} \mathrm{prod}(c)\!\!\! \sum_{s \in S_{-\beta}(c)}\!\!  \mathrm{prod}(s) g(X) X^k \Big ] \\
    &= \sum_{k = 0}^{n} \Big (\! \sum_{c \in C(B, k)} \mathrm{prod}(c)\!\!\! \sum_{s \in S_{-\beta}(c)}\!\!\! \E[ \mathrm{prod}(s) ] \Big ) \E \big [ g(X) X^k \big ]
\end{align*}
where we have pulled $\mathrm{prod}(c)$ out since it is a constant for each $c$, and we have pulled out $\E[ \mathrm{prod}(s) ]$ since $\mathrm{prod}(s)$ is either a constant (some product of $\alpha$ factors) or a random variable independent of $X$ (if it contains any $\epsilon$ factors).

In general, $\E[ \mathrm{prod}(s) ]$ is proportional to the expected product of a (possibly empty) set of zero mean, jointly Gaussian variables $\epsilon_j\in s$. It can be evaluated using Isserlis' theorem, which reduces to a sum of products of covariances between residuals. The covariance between two residuals $\epsilon_i$, $\epsilon_j$ is
\begin{align*}
    &\mathrm{Cov}(\epsilon_i, \epsilon_j) \\ &= \mathrm{Cov} \Big [ \alpha_i + \beta_i X - Y_i, \alpha_j + \beta_j X - Y_j \Big ] \\
    &= \mathrm{Cov}(Y_i, Y_j) \cancel{- \beta_i \mathrm{Cov}(X, Y_j) - \beta_j \mathrm{Cov}(X, Y_i)} + \beta_i \beta_j \mathrm{Var}(X) \\
    %&= \mathrm{Cov}(X, Y_i) \mathrm{Cov}(X, Y_j) \mathrm{Var}(X)^{-1} - \beta_i \mathrm{Cov}(X, Y_j) - \beta_j \mathrm{Cov}(X, Y_i) + \mathrm{Cov}(Y_i, Y_j) \\
    %&= \beta_i \beta_j \mathrm{Var}(X) - \beta_i \mathrm{Cov}(X, Y_j) - \beta_j \mathrm{Cov}(X, Y_i) + \mathrm{Cov}(Y_i, Y_j) 
    %&= \mathrm{Cov}(Y_i, Y_j) - \beta_i \beta_j \mathrm{Var}(X) \\
    &= \mathrm{Cov}(Y_i, Y_j) - \mathrm{Cov}(X, Y_i) \mathrm{Cov}(X, Y_j) \mathrm{Var}(X)^{-1},
\end{align*}
using the fact that $\beta_i = \frac{\mathrm{Cov}(X, Y_i)}{\mathrm{Var}(X)}$ and $\beta_j = \frac{\mathrm{Cov}(X, Y_j)}{\mathrm{Var}(X)}$.

We are now in a position to derive a formula for fixed $n$. When $n = 2$, this simplifies to
\begin{align}
    \E \Big [g(X) Y_1 Y_2 \Big ] &= \beta_1 \beta_2 \E \Big [ g(X) X^2 \Big] \\ \nonumber
    &+ \big(\alpha_1 \beta_2 + \alpha_2 \beta_1 \big) \E \Big [g(X)X\Big ] \\ \nonumber
    &+ \alpha_1 \alpha_2 \mathrm{Cov}(\epsilon_1, \epsilon_2) \E \Big [g(X) \Big].
\end{align}
\end{proof}

\subsection{Gaussian mixture inputs}
\label{sec:mixture}

So far we have assumed the input takes on a Gaussian distribution. But in many cases it makes sense to model the input as being drawn from a mixture distribution with $k$ distinct components, perhaps corresponding to different class labels. Luckily, it turns out that we can extend the above derivation to the case where the input has a Gaussian mixture distribution by applying the law of total covariance. Recall that for any cross-covariance matrix $\SSigma_{XY}$ where $X$ and $Y$ follow a mixture distribution indexed by $Z \sim \mathrm{Cat}(k)$, we have
\begin{equation}\label{eq:total-cov}
    \SSigma_{XY} = \mathbb{E}[\SSigma_{XY \mid Z}] + \SSigma_{\mathbb{E}[X \mid Z], \mathbb{E}[Y \mid Z]}
\end{equation}
where $\SSigma_{\mathbb{E}[X \mid Z], \mathbb{E}[Y \mid Z]}$ denotes the cross-covariance matrix of the conditional means of $X$ and $Y$. By setting $Y = X$, we can use Eq.~\ref{eq:total-cov} to solve for the covariance matrix of the inputs, and then by setting $Y = f(X)$ we can solve for the cross-covariance matrix of the input and the MLP output, using the analytic formulas we've already derived.

\subsection{Gated linear units}

We can now efficiently compute polynomial approximations to gated linear units (GLUs) as well as MLPs. It's easy to see from Eq.~\ref{eq:glu} that, given Gaussian inputs, the preactivations $\yy := \xx W + b$ and $\zz := \xx V + c$ are jointly Gaussian. Each component of $\E[\mathrm{GLU}_\phi(\xx, W, V, b, c)]$ is then an integral of the form $\E[\phi(Y_i) Z_i]$, which can be evaluated with Stein's lemma. Each entry of the cross-covariance matrix takes the form $\E[\phi(Y_i) Z_i X_j]$, allowing us to apply Theorem~\ref{thm:master}. Similar arguments apply, \emph{mutatis mutandis}, to higher-order polynomial approximants.

\section{Neural Networks Learn Polynomials of Increasing Degree}\label{sec:functional-complexity}

Prior work suggests that, across training time, the complexity of the function represented by a neural network tends to increase \citep{nakkiran2019sgd}, and networks tend to exploit statistical moments in increasing order as training progresses \citep{belrose2024neural}. In this section, we apply the polynomial approximation machinery derived above to the following question: \emph{Do neural networks learn polynomial functions of increasing degree?}

Specifically, our hypothesis is that the $R^2$ of the least-squares linear approximation to an MLP should start out high, and decrease nearly monotonically with training time. Then, at some point \emph{after} the $R^2$ of the linear approximant starts to decrease, we should see the $R^2$ of the least-squares quadratic approximant to decrease, again monotonically.

\subsection{Methods}

We consider the setting of image classification, although in principle the derivation is domain-agnostic. We train a single hidden layer MLP with ReLU activation on the MNIST dataset \citep{lecun1998gradient} using schedule-free AdamW \citep{defazio2024road} with 1K warmup steps, a batch size of $64$, and weight decay of $0.1$, saving checkpoints at log-spaced intervals. MNIST is an ideal dataset for this task since it is known to be very well-modeled as a Gaussian mixture distribution. Samples from a Gaussian mixture fit to MNIST are difficult for a human to distinguish from real samples \citep[Figure 2]{belrose2023leace}.

For each checkpoint, we fit least-squares linear and quadratic approximants under the assumption of Gaussian mixture inputs whose means and covariances match those of the MNIST classes.\footnote{For the quadratic approximants, we found that the analytic solution for Gaussian mixture inputs is computationally intractable on MNIST. We therefore fit the approximant under the assumption of $\mathcal{N}(0, 1)$ inputs and finetune it on Gaussian mixture samples using SGD. This accounts for the noise visible in the plots for quadratic FVU. For implementation details, see Appendix~\ref{app:quadratic-approx}.} We report the fraction of variance unexplained (FVU) of these checkpoints in Figure~\ref{fig:fvu_comparison}.

% One worry about the least-squares objective is that it does not respect the probabilistic interpretation of the output logits for these image classifiers. To test the robustness of our results, we finetune copies of the linear and quadratic approximants using schedule-free SGD \citep{defazio2024road} on Gaussian mixture samples with a KL divergence loss function, and report the resulting KLs in Figure~\ref{fig:kl_comparison}.

\subsection{Quantitative results}

Between 500 and 1K training steps, there is a clear phase transition where the FVU and KL divergence for the linear approximant start to increase sharply, and continue to do so until around 4K steps. Meanwhile, the FVU for the quadratic approximant is nearly constant over this same period. We interpret this as a ``quadratic phase'' of training wherein the network learns to exploit second-order statistical information in the input.

Interestingly, we also find that in the first few hundred steps of training, the FVU for both linear and quadratic approximants goes \emph{down}, indicating that the network is actually getting simpler during this time. While unexpected, this makes some amount of sense: while randomly initialized neural networks tend to be simple \citep{teney2024neural}, it may make sense for SGD to eliminate noise from the network early in training before making it more complex. On the other hand, we find that this phase disappears when we use KL divergence instead of FVU to measure the discrepancy between the network and its polynomial approximants (\cref{fig:kl_comparison}).

%Finally, there seems to be a period after about $10^4$ steps where the FVU and KL for the quadratic approximant increases more rapidly, while the FVU for the linear approximant stagnates and actually \emph{decreases} slightly. This may be viewed as a ``higher-order phase'' in which the network learns features that are difficult to approximate with a first- or second-degree polynomial, although it is less distinct than the quadratic phase on this dataset.% These two phases are broadly in line with what we would expect from the distributional simplicity bias literature.

% Interestingly, both the linear and quadratic KL divergences increase in conjunction, contradicting previous beliefs (citation needed) that models learn in "stages." We also observe a form of double descent where, in both the linear and quadratic fits, the KL divergence increase in two stages, with a small rebound.

% We performed a variety of ablations to test how 'real' this phenomenon is: \textcolor{red}{Describe ablations and results here.}. We expand on those results further in the appendix.

\begin{figure}[t]
    \centering
    \includegraphics[width=0.45\textwidth]{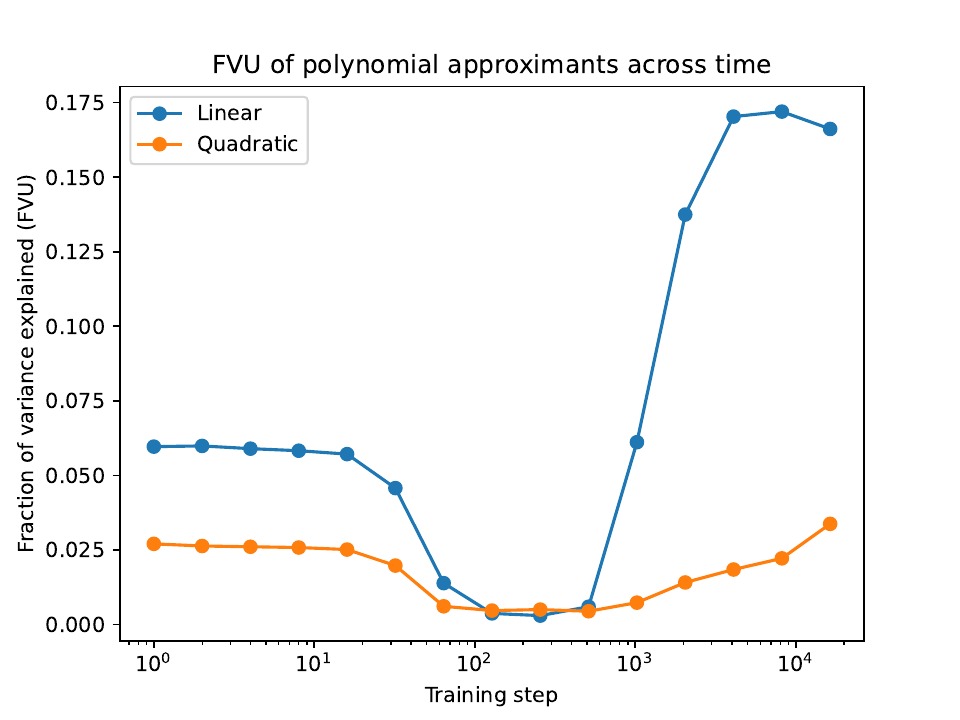}
    \caption{Fraction of variance unexplained for linear and quadratic approximants on a Gaussian mixture distribution imitation the MNIST training set. There is a sharp increase in the linear FVU between 500 and 1K training steps, while the quadratic FVU is roughly constant over the same time period.}
    \label{fig:fvu_comparison}
\end{figure}
\begin{figure}[h]
    \centering
    \includegraphics[width=0.45\textwidth]{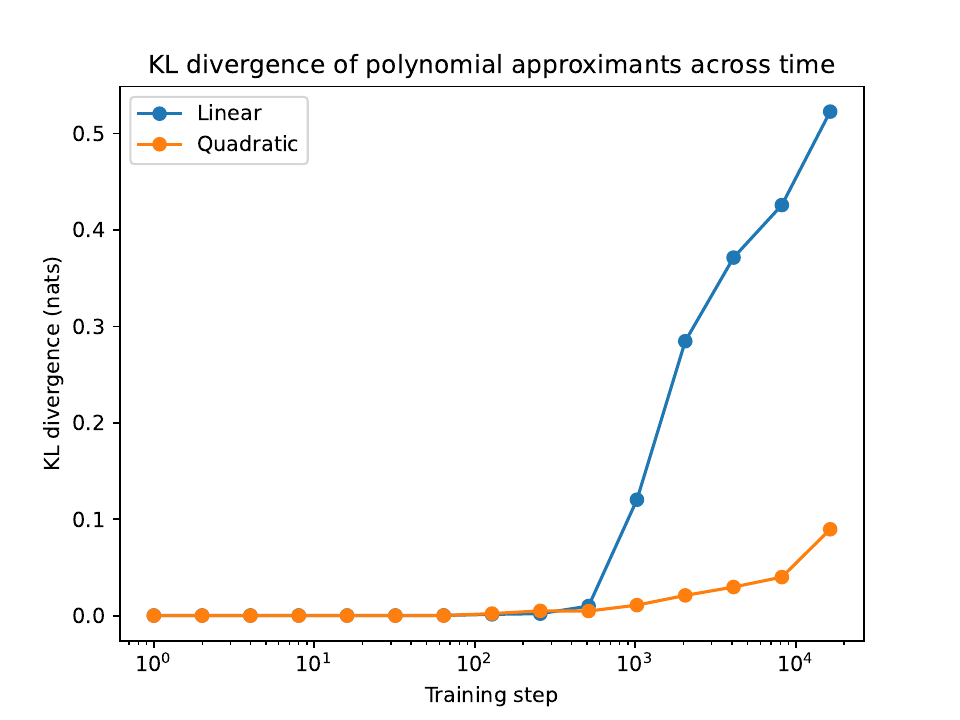}
    \caption{KL divergence of linear and quadratic approximants from the network on which they were fit, evaluated on a Gaussian mixture distribution imitating the MNIST training set. The trend mirrors the FVU plot (\cref{fig:fvu_comparison}) except that the KL does not decrease in the first few hundred steps before increasing.}
    \label{fig:kl_comparison}
\end{figure}

% \begin{figure}[ht]
%     \centering
%     \begin{minipage}{0.23\textwidth}
%         \centering
%         \includegraphics[width=\textwidth]{figures/measures/fvu_old.pdf}
%     \end{minipage}
%     \hfill
%     \begin{minipage}{0.23\textwidth}
%         \centering
%         \includegraphics[width=\textwidth]{figures/measures/kl_old.pdf}
%     \end{minipage}
%     %\includegraphics[width=0.46\textwidth]{icml2025/figures/eigvals_and_acts/probe_singular_values.png} % Adjust 'example.png' to your file name
%     \caption{FVU and KL Divergence for linear and quadratic approximants on the MNIST test set and on $\mathcal{N}(0,1)$.} % Add an image caption
%     \label{fig:probe_singular_values} % Add a label for referencing
% \end{figure}

\subsection{Qualitative feature visualization}

To complement our quantitative analysis, we visualize the top eigenvector of the quadratic approximant for a random class, as well as the coefficients of the linear approximant corresponding to that class, over the course of an extra-long training run in \cref{fig:qualitative}. We see that linear and quadratic features are most intuitively interpretable between steps $2^{12}$ and $2^{14}$, after which the model starts to overfit.

\begin{figure*}[t]
    \centering
    \includegraphics[width=1.0\textwidth]{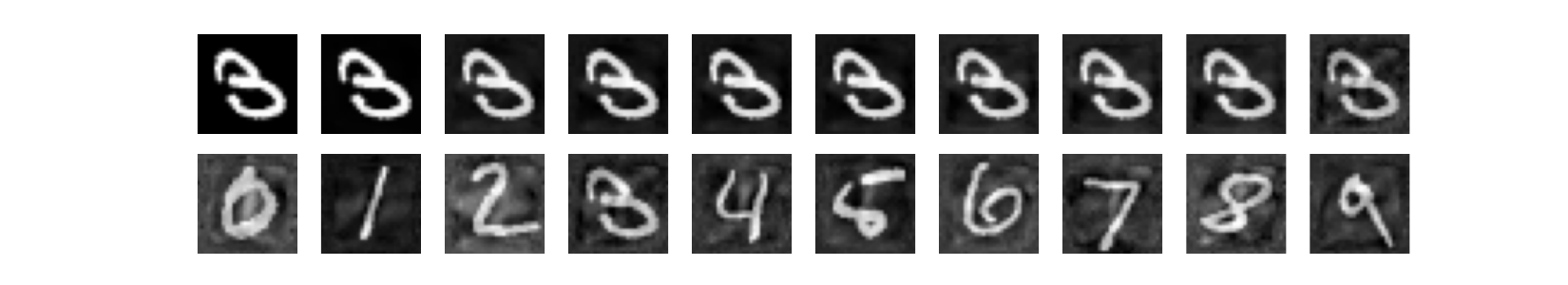}
    \caption{\textbf{Top row:} Adversarial `3' for different intervention strengths, ranging from one to ten SVD components ablated (left to right). \textbf{Bottom row:} Random examples of each digit with all ten SVD components ablated. Examples bear high resemblance to the original images, despite being unclassifiable by the MLP.}
    \label{fig:svd_attack_visualization}
\end{figure*}

\begin{figure}[h]
    \centering
    \includegraphics[width=0.45\textwidth]{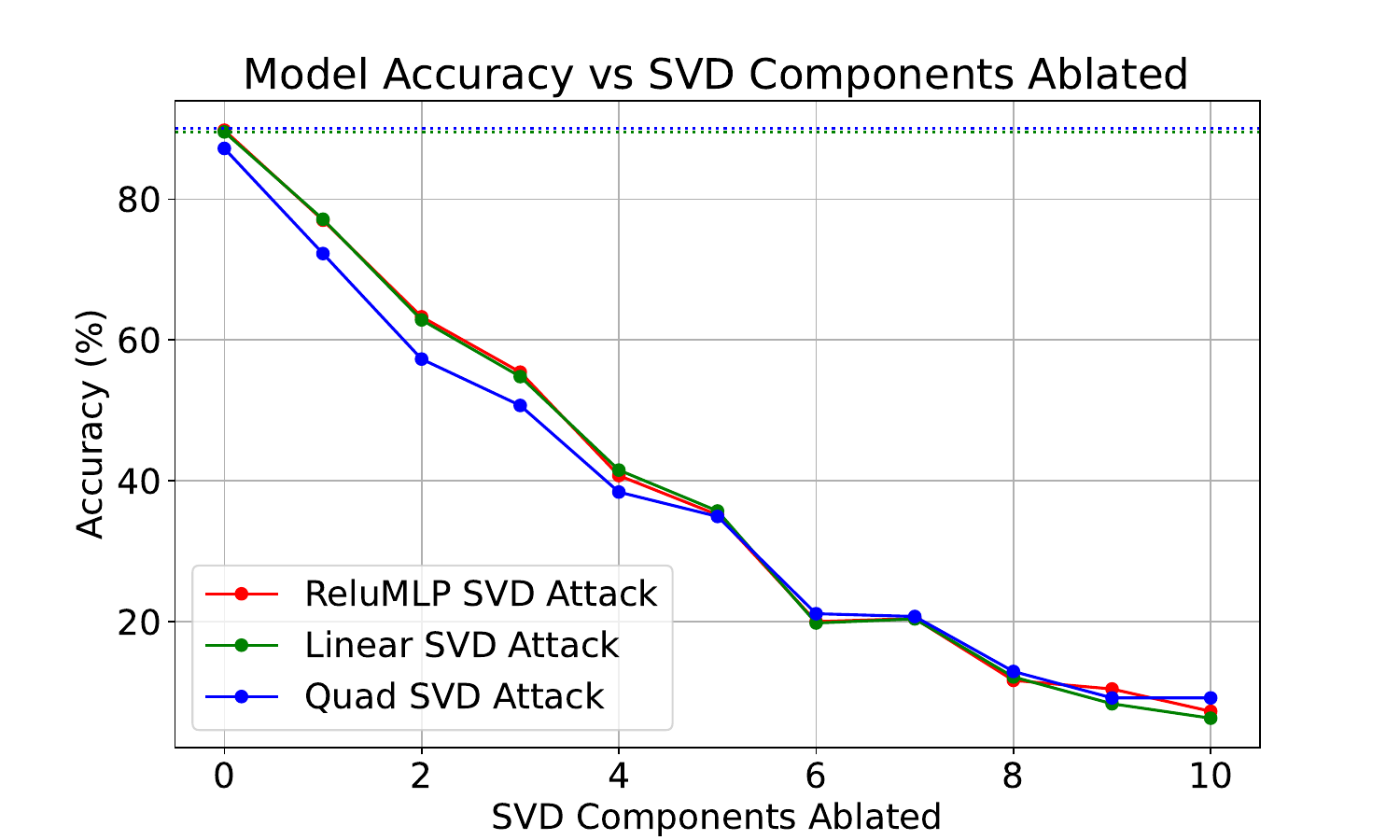}
    \caption{Steering with an adversarial mask. Strikingly, the original network's accuracy drops in perfect lockstep with that of its linear and quadratic approximants. Just four SVD components are needed to bring accuracy below 50\%, and ablating the entire rowspace (only 10 out of 784 total dimensions) sufficient to make the MLP perform no better than random chance. This indicates that the approximations are appropriately capturing the generalization behavior of the MLP.}
    \label{fig:7eigs}
\end{figure}

\subsection{Adversarial attacks}
Linear and quadratic networks can be decomposed into importance terms using SVD and eigendecomposition, then used to construct demonstrably effective adversarial attacks. If these attacks on polynomial approximants transfer to the ReLU network, it would provide evidence that our approximants capture the out-of-distribution behavior of the network. 

We formulate interventions as input space transformations, producing one transformation that gets uniformly applied to all examples in a test set. To measure the efficacy of an intervention, we compute the same quantitative measures as before, evaluating accuracy on the test set.

\paragraph{SVD attack.} We compute the singular value decomposition of the linear coefficients $$\beta = U \Sigma V^T= \sum_{i=1}^{10} \sigma_i u_i v^T$$ with $U: \mathbb{R}^\text{input} \to \mathbb{R}^{10}$ and $V: \mathbb{R}^{10} \to \mathbb{R}^{10}$.  We neutralize the top $k$ SVD components from the MLP, by applying an orthogonal projection $P_k$ to all its inputs: $$P_k = I - \sum_{i=1}^k u_i u_i^T$$

\paragraph{Results.}
% We repeat the quantitative measures as before, checking accuracy of the model and its approximants, and the approximant FVU and KL divergence on the test set vs the Gaussian prior. For quadratic steering, we show two baselines, no steering and adding gaussian noise.

In \cref{fig:7eigs}, we see that the ReLU model's performance drops in close lockstep with the approximants' performance, and after ablating just four SVD components, its accuracy is already under 50\%. Projecting the images onto the nullspace of the linear approximant (i.e. ablating all ten SVD components) brings the MLP to near-random performance, while the images are still human-intelligible (\cref{fig:svd_attack_visualization}, bottom row). This is a striking confirmation that our linear approximation captures features of the input that are causally important for the original model's performance.

% In quadratic steering, we see that the adversarial steering is effective, but not as effective as Gaussian noise for impacting performance. On the other hand, at all intervention levels, KL divergence and FVU for quadratic steering remains negligible, whereas Gaussian noise increases the gap. %, all measures are tightly bound to the ReLU MLP's performance

\section{Conclusion}
\paragraph{Summary.} We derive analytic formulas for polynomial least-squares approximants for single hidden layer MLPs and gated linear units, given Gaussian inputs, for a variety of activation functions. We show empirically that these approximants can have surprisingly high $R^2$, with quadratics explaining well over 95\% of the variance in MLP outputs, when the model is trained on MNIST. Since the explained variance is so high, mechanistic insights from the approximations transfer to the ReLU network, which we demonstrate by constructing provably effective adversarial examples from features of the linear and quadratic approximations, and applying them to the original MLP.

\paragraph{Limitations.} Our analytic derivations make the assumption that the input has some Gaussian mixture distribution, and it is unclear if our derivation can be generalized in a useful way to other distributions. While we argue that the Gaussianity assumption is independently motivated, since it prevents the approximant from overfitting to the specific properties of a training dataset, it may not be suitable for all applications. In cases where Gaussianity is too restrictive, we recommend initializing a gradient-based optimizer with our analytic polynomial approximant, then finetuning on samples from the desired target distribution.

Computing the degree $n$ polynomial approximation for an MLP with input dimension $k$ and $d$ neurons involves materializing a tensor with $dk^n$ parameters. In a world where intermediate activations can have dimensionality in the tens of thousands, $n = 2$ is the highest that can feasibly be computed analytically. Therefore, if the best quadratic approximation for an MLP is poor, there is little that can be done to interpret the model using our approach in isolation. % initialization scheme. %for example to train sparse transcoders with a linear skip connection optimal for reconstruction. %be to initialize
  %, it is unclear how to adapt our approximations to condition with higher specificity. 
%we have no methods to incorporate that into the approximation. Our approach currently only applies to the gaussian prior $\mathcal{N}(\mu,\Sigma)$, and we don't have methods to condition on a higher specificity.

% \paragraph{Related Work.} This work connects to several areas of research interpretability and theory of deep learning. Our approach aligns with prior investigations into the geometry of ReLU activation regions, such as the polytope lens \citep{black2022interpretingneuralnetworkspolytope} and studies on activation patterns of ReLU networks \citep{hanin2019deeprelunetworkssurprisingly}. On the interpretability front, this work contrasts with data and compute-heavy methods that disentangle polysemantic neurons using sparse autoencoders or sparse transcoders \citep{huben2023sparse,dunefsky2024transcodersinterpretablellmfeature}. Finally, our findings relate to broader discussions of training dynamics, including simplicity bias \citep{refinetti2023neural} and work on progressive feature learning in neural networks \citep{belrose2023leace,belrose2024neural}.

\paragraph{Future Work.} While we focus on simple image classification models in this paper, we think that our polynomial approximation schemes could be useful for interpreting feedforward modules in deep networks, like transformers. Performing eigendecomposition on a quadratic approximation to a transformer FFN provides us with an overcomplete basis of $d^2$ eigenvectors for the residual stream of dimension $d$. It is known that the singular vectors of FFN weight matrices can often be interpreted by projecting them into next-token prediction space using the unembedding matrix \citep{millidge2022singular}, so it is plausible that something similar could be done for these basis vectors. Since the basis would be overcomplete, however, it may enable us to overcome the polysemanticity problem in a similar way to sparse autoencoders \citep{huben2023sparse}. It may also turn out that more sophisticated tensor decomposition methods enable us to extract a more useful overcomplete basis than simple eigendecomposition.

% Acknowledgements should only appear in the accepted version.
\section*{Contributions and Acknowledgements}

Nora Belrose derived the analytic formulas for approximating MLPs with polynomial functions, and did most of the writing. Alice Rigg performed the adversarial attack experiments and generated Figure 1. Nora and Alice are funded by a \href{https://www.openphilanthropy.org/grants/eleuther-ai-interpretability-research/}{grant} from Open Philanthropy. We thank Coreweave for computing resources.

\section*{Code Availability}

Alongside this paper, we release a NumPy and PyTorch library for analytically computing polynomial approximants of MLPs and GLUs at \url{https://github.com/EleutherAI/polyapprox}.

\section*{Impact Statement}

This paper presents work whose goal is to advance the field of Interpretability. There are many potential societal consequences of our work, none which we feel must be specifically highlighted here.

% In the unusual situation where you want a paper to appear in the
% references without citing it in the main text, use \nocite
%\nocite{langley00}

\bibliography{citations}
\bibliographystyle{icml2025}

%%%%%%%%%%%%%%%%%%%%%%%%%%%%%%%%%%%%%%%%%%%%%%%%%%%%%%%%%%%%%%%%%%%%%%%%%%%%%%%
%%%%%%%%%%%%%%%%%%%%%%%%%%%%%%%%%%%%%%%%%%%%%%%%%%%%%%%%%%%%%%%%%%%%%%%%%%%%%%%
% APPENDIX
%%%%%%%%%%%%%%%%%%%%%%%%%%%%%%%%%%%%%%%%%%%%%%%%%%%%%%%%%%%%%%%%%%%%%%%%%%%%%%%
%%%%%%%%%%%%%%%%%%%%%%%%%%%%%%%%%%%%%%%%%%%%%%%%%%%%%%%%%%%%%%%%%%%%%%%%%%%%%%%
\newpage
\appendix
\onecolumn
\section{Integrals}

\subsection{Gaussian Linear Unit (GELU)}

Recall that the GELU activation function is defined as $x \Pphi(x)$, where $\Pphi$ denotes the standard normal CDF. This allows us to employ known results from the integral tables of \citet{owen1980table}.

\subsubsection{Mean}

Specifically we can make use of the identity
\begin{equation}
    \int_{-\infty}^{\infty} x \Pphi(a + bx) \varphi(x) dx = \frac{b}{\sqrt{1 + b^2}} \varphi \Big (\frac{a}{\sqrt{1 + b^2}} \Big )
\end{equation}
where $\varphi$ denotes the standard normal PDF, to show that
\begin{align}
    \E_x [\mathrm{GELU}(x)] &= \int_{-\infty}^{\infty} (\mu + \sigma z) \Pphi(\mu + \sigma z) \varphi(z) dz \\
    &= \mu \int_{-\infty}^{\infty} \Pphi(\mu + \sigma z) \varphi(z) dz + \sigma \int_{-\infty}^{\infty} z \Pphi(\mu + \sigma z) \varphi(z) dz \\
    &= \mu \Pphi \Big ( \frac{\mu}{\sqrt{1 + \sigma^2}} \Big ) + \frac{\sigma^2}{\sqrt{1 + \sigma^2}} \varphi \Big ( \frac{\mu}{\sqrt{1 + \sigma^2}} \Big ).
\end{align}

\subsubsection{Derivative}

By the product rule, the first derivative of GELU is
\begin{align}
    \frac{d}{dx} \mathrm{GELU}(x) &= \frac{d}{dx} \Big [ x \Pphi(x) \Big ] = \Pphi(x) + x \varphi(x)
\end{align}
The expected derivative under $\mathcal{N}(\mu, \sigma)$ is then
\begin{align*}
    \E_x [\mathrm{GELU}'(x)] &= \int_{-\infty}^{\infty} \Pphi(\mu + \sigma z) \varphi(z) dz + \int_{-\infty}^{\infty} (\mu + \sigma z) \varphi(\mu + \sigma z) \varphi(z) dz \\
    &= \Pphi \Big ( \frac{\mu}{\sqrt{1 + \sigma^2}} \Big ) + \mu \int_{-\infty}^{\infty} \varphi(\mu + \sigma z) \varphi(z) dz + \sigma \int_{-\infty}^{\infty} z \varphi(\mu + \sigma z) \varphi(z) dz \\
    &= \Pphi \Big ( \frac{\mu}{\sqrt{1 + \sigma^2}} \Big ) + \frac{\mu}{\sqrt{1 + \sigma^2}} \varphi \Big ( \frac{\mu}{\sqrt{1 + \sigma^2}} \Big ) - \sigma \varphi \Big ( \frac{\mu}{\sqrt{1 + \sigma^2}} \Big ) \frac{\mu \sigma}{(1 + \sigma^2)^{3/2}} \\
    &= \Pphi \Big ( \frac{\mu}{\sqrt{1 + \sigma^2}} \Big ) + \varphi \Big ( \frac{\mu}{\sqrt{1 + \sigma^2}} \Big ) \Big [ \frac{\mu}{\sqrt{1 + \sigma^2}} - \frac{\mu \sigma^2}{(1 + \sigma^2)^{3/2}} \Big ].
\end{align*}

\subsubsection{Higher-order moments}

We use the following identity from \citet{owen1980table}:
\begin{equation}\label{eq:gelu-nct}
    \int_{0}^{\infty} z^n \Pphi(a z + b) \varphi(z) dz = \frac{\Gamma(\frac{n + 1}{2}) 2^{(n - 1)/2}}{\sqrt{2 \pi}} F(a \sqrt{n + 1}; -b, n + 1)
\end{equation}
where $F(a \sqrt{n + 1}; -b, n + 1)$ is the cumulative distribution of the noncentral Student's $t$ distribution with $n + 1$ degrees of freedom and noncentrality parameter $-b$. The formula for this CDF is complex, but an efficient implementation is available in the SciPy \texttt{stats} module \href{https://docs.scipy.org/doc/scipy/tutorial/stats/continuous_nct.html}{here}. To convert Eq.~\ref{eq:gelu-nct} into an integral over the entire real line, we can use the identity
\begin{equation}\label{eq:pos-neg-int}
    \int_{-\infty}^{0} z^n \Pphi(a z + b) \varphi(z) dz = (-1)^n \int_{0}^{\infty} z^n \Pphi(-a z + b) \varphi(z) dz,
\end{equation}
and add together the positive and negative ``parts'' of the integral.

Finally, substituting $x = \sigma z + \mu$ and applying the binomial theorem, we have
\begin{align}
    \E_x \big [ x^n \Pphi(x) \big ] &= \int_{-\infty}^{\infty} (\sigma z + \mu)^n \Pphi(\sigma z + \mu) \varphi(z) dz\\
    &= \sum_{k = 0}^{n + 1} \binom{n + 1}{k} \mu^{n + 1 - k} \sigma^k \int_{-\infty}^{\infty} z^k \Pphi(\sigma z + \mu) \varphi(z) dz
\end{align}
We can now plug in Eqs.~\ref{eq:gelu-nct} and~\ref{eq:pos-neg-int} to solve for the expectation.

\subsection{Rectified Linear Unit (ReLU)}\label{app:relu}

The piecewise linear nature of ReLU makes the derivation of integrals involving it fairly straightforward. Essentially, we need to compute the probability of landing in the positive part of the ReLU's domain, then compute a simple Gaussian integral over this part of the domain.

\subsubsection{Mean}

We begin with the expectation $\E_x [\mathrm{ReLU}(x)]$ for $x \sim \mathcal{N}(\mu, \sigma)$. We first reparametrize $x = \mu + \sigma z$ where $z \sim \mathcal{N}(0, 1)$, and consider the equivalent expectation $\E_z [\mathrm{ReLU}(\mu + \sigma z)]$. Note that $\mathrm{ReLU}(x) > 0$ if and only if $z > -\frac{\mu}{\sigma}$. Now we may split the integral into two parts, corresponding to the positive and zero parts of the ReLU:
\begin{align*}
    \E_z [\mathrm{ReLU}(\mu + \sigma z)] &= \int_{-\frac{\mu}{\sigma}}^{\infty} (\mu + \sigma z) \varphi(z) dz + \cancel{\int_{-\infty}^{-\frac{\mu}{\sigma}} 0\: \varphi(z) dz} \\
    &= \mu \int_{-\frac{\mu}{\sigma}}^{\infty} \varphi(z) dz + \sigma  \int_{-\frac{\mu}{\sigma}}^{\infty} z \varphi(z) dz
\end{align*}
We can evaluate the first term using the standard normal CDF; by symmetry, it is simply $\Pphi ( \frac{\mu}{\sigma})$.

To evaluate the second term, we can use the fact that $\int_a^{\infty} z \varphi(z) dz = \varphi(a)$ for any $a$. To see this, note that
\begin{equation}
\frac{d}{dz} \varphi(z) = \frac{d}{dz} \Big ( \frac{1}{\sqrt{2 \pi}} \exp \big ( -\frac{z^2}{2} \big ) \Big ) = \frac{1}{\sqrt{2 \pi}} (-z) \exp \Big ( -\frac{z^2}{2} \Big ) = -z \varphi(z)
\end{equation}
and therefore
\begin{equation}\label{eq:linear-expectation}
    \int_a^{\infty} z \varphi(z) dz = -\int_a^{\infty} \frac{d}{dz} \varphi(z) dz = -\big [ \varphi(\infty) - \varphi(a) \big ] = \varphi(a).
\end{equation}
Applying this identity to our case yields $\varphi(-\frac{\mu}{\sigma})$, or by the symmetry of the standard normal PDF about zero, $\varphi(\frac{\mu}{\sigma})$.

Putting everything together, we have
\begin{equation}\label{eq:relu-expectation}
    \E_x [\mathrm{ReLU}(x)] = \mu \Pphi \Big ( \frac{\mu}{\sigma} \Big ) + \sigma \varphi \Big (\frac{\mu}{\sigma} \Big ).
\end{equation}

We can apply the above formula to evaluate
\begin{equation}
    \E_{\zz}[\mathrm{ReLU}(\mathbf{A}\zz + \mathbf{b})]_i = b_i \Pphi \Big ( \frac{b_i}{\| \mathbf{A}_i \|} \Big ) + \| \mathbf{A}_i \| \varphi \Big ( \frac{b_i}{\| \mathbf{A}_i \|} \Big )
\end{equation}
Let $\boldsymbol{s}$ denote the vector containing the Euclidean norms of the rows of $\mathbf{W}_1$. For a whole MLP it would then be
\begin{equation}
    \E_{\zz}[f(\zz)] = \mathbf{W}_2 \Big [ \mathbf{b}_1 \odot \Pphi (\mathbf{b}_1 \odot \boldsymbol{s}^{-1} ) + \boldsymbol{s} \odot \varphi ( \mathbf{b}_1 \odot \boldsymbol{s}^{-1} ) \Big ]
\end{equation}

\subsection{Higher-order moments}

The above derivation is special case of a more general formula for $\E[ x^n \mathrm{ReLU}(x) ]$ for any $n \geq 0$. Reparametrizing as before and applying the binomial theorem, we have
\begin{align}
    \E[ x^n \mathrm{ReLU}(x) ] &= \int_{-\infty}^{\infty} (\mu + \sigma z)^n \mathrm{ReLU}(\mu + \sigma z) \varphi(z) dz \\
    &= \sum_{k = 0}^{n + 1} \binom{n + 1}{k} \mu^{n + 1 - k} \sigma^k \int_{-\frac{\mu}{\sigma}}^\infty z^k \varphi(z) dz. \label{eq:relu-polynomial}
\end{align}
For $k \geq 2$, the integral $\int_{-\frac{\mu}{\sigma}}^\infty z^k \varphi(z) dz$ can be evaluated using the following recursion:
\begin{equation}
    \int_{-\frac{\mu}{\sigma}}^\infty z^k \varphi(z) dz = \Big ( \frac{\mu}{\sigma} \Big )^{k - 1} \varphi \Big ( \frac{\mu}{\sigma} \Big ) + (k - 1) \int_{-\frac{\mu}{\sigma}}^\infty z^{k - 2} \varphi(z) dz,
\end{equation}
where the value for $k = 0$ is simply $\Pphi(\frac{\mu}{\sigma})$, and the value for $k = 1$ is  $\varphi(\frac{\mu}{\sigma})$ (see Eq.~\ref{eq:linear-expectation}). The recursion is computationally efficient because Eq.~\ref{eq:relu-polynomial} makes use of every intermediate value from $k = 0$ to $k = n + 1$.

Alternatively, we can use the upper \href{https://en.wikipedia.org/wiki/Incomplete_gamma_function}{incomplete gamma function}, which is available in popular libraries like \href{https://docs.scipy.org/doc/scipy/reference/generated/scipy.special.gammainc.html}{SciPy}, \href{https://pytorch.org/docs/stable/special.html#torch.special.gammaincc}{PyTorch}, and \href{https://jax.readthedocs.io/en/latest/_autosummary/jax.scipy.special.gammainc.html}{JAX}. It is defined as $\Gamma(s, x) = \int_x^{\infty} t^{s - 1} \exp(-t) dt$. Specifically,
\begin{align}
    \int_{-\frac{\mu}{\sigma}}^\infty z^k \varphi(z) dz = \frac{1}{\sqrt{2 \pi}} \int_{-\frac{\mu}{\sigma}}^\infty z^k \exp \Big ( -\frac{z^2}{2} \Big ) dz = \frac{2^{k/2}}{\sqrt{2 \pi}} \Gamma \Big ( \frac{k + 1}{2}, -\frac{\mu^2}{2 \sigma^2} \Big ).
\end{align}
Putting it all together, we have
\begin{equation}
    \E[ x^n \mathrm{ReLU}(x) ] = \frac{1}{\sqrt{2 \pi}} \sum_{k = 0}^{n + 1} \binom{n + 1}{k} 2^{k/2} \mu^{n + 1 - k} \sigma^k \Gamma \Big ( \frac{k + 1}{2}, -\frac{\mu^2}{2 \sigma^2} \Big ).
\end{equation}

\section{Implementation details for quadratic approximants}\label{app:quadratic-approx}

Computing a quadratic approximant in closed form requires solving a linear system using the covariance matrix of the features $\zz = [\xx, \phi_2(\xx)]$, where $\phi_2$ is defined in Eq.~\ref{eq:quadratic-map}. The dimension of $\zz$ is precisely $\frac{d(d + 1)}{2}$, where $d$ is the dimensionality of $\xx$. Even for a simple dataset like MNIST, of dimension $28 \times 28 = 784$, this is already $\frac{784 \times 785}{2} = 307\:720$. The covariance matrix of $\zz$ thus has $9.47 \times 10^{10}$ entries, many more than can fit in 32-bit precision on an H100 graphics card. And since the time complexity of solving a linear system (with commonly used algorithms) scales cubically in the number of rows in the matrix, the overall complexity is $O(d^6)$. This is intractable, unless $\mathrm{Cov}(\zz)$ has special structure-- which it does, if $\xx \sim \mathcal{N}(0, 1)$. In that case, $\mathrm{Cov}(\zz)$ is a diagonal matrix with entries are in $\{1, 2\}$ according to a simple pattern.

For this reason, we chose to initialize our quadratic approximants with coefficients computed under the assumption that $\xx$ is a standard Gaussian vector. We then finetuned them using schedule-free SGD \citep{defazio2024road} on batches of samples from a Gaussian mixture distribution whose means and covariance matrices matched those of the classes in MNIST. Since the objective is convex, the resulting coefficients should be excellent approximations of the true least-squares values.

\end{document}